\documentclass[11pt]{article}

\title{An Impossibility Theorem for Node Embedding}
\date{}

\def\isArXiVversion{} 
\author{%
  T.~Mitchell~Roddenberry, Yu~Zhu, Santiago~Segarra \\
  Department of Electrical and Computer Engineering \\
  Rice University \\
  \href{mailto:mitch@rice.edu}{\texttt{mitch}},\href{mailto:yz126@rice.edu}{\texttt{yz126}},\href{mailto:segarra@rice.edu}{\texttt{segarra}}\texttt{@rice.edu}%
}

\usepackage[utf8]{inputenc}
\usepackage[T1]{fontenc}
\usepackage{hyperref}
\usepackage{url}
\usepackage{doi}
\usepackage{booktabs}
\usepackage{amsmath,amssymb,amsthm,amsfonts,marvosym,nicefrac}
\usepackage{multirow}
\usepackage{nicefrac}
\usepackage{microtype}
\usepackage{xcolor}
\usepackage{soul}

\ifdefined\isArXiVversion
\usepackage{fullpage}
\usepackage{microtype}
\usepackage[
hyperref=true,
url=false,
style=alphabetic,
maxbibnames=100,
]{biblatex}
\else
\usepackage[nonatbib]{neurips_2021}
\usepackage[
hyperref=true,
url=true,
style=alphabetic,
maxbibnames=100,
]{biblatex}
\fi

\usepackage{doi}

\hypersetup{
  colorlinks = true,
  allcolors = blue,
}
\usepackage[capitalise,nameinlink]{cleveref}

\usepackage{tikz}
\usetikzlibrary{arrows}
\usetikzlibrary{hobby}
\usetikzlibrary{calc}

\newenvironment{subproof}[1][\proofname]{%
  \begin{proof}[#1]%
}{%
  \end{proof}%
}

\makeatletter
\newtheorem*{rep@theorem}{\rep@title}
\newcommand{\newreptheorem}[2]{%
\newenvironment{rep#1}[1]{%
 \def\rep@title{#2 \@refstar{##1}}%
 \begin{rep@theorem}}%
 {\end{rep@theorem}}}
\makeatother

\newtheorem{prop}{Proposition}
\newtheorem{theorem}{Theorem}
\newreptheorem{theorem}{Theorem}
\newtheorem{lemma}{Lemma}

\newtheoremstyle{myremark}
{\topsep} 
{\topsep} 
{\normalfont} 
{} 
{\bfseries} 
{.} 
{5pt plus 1pt minus 1pt} 
{\thmname{#1}\thmnumber{ #2}\thmnote{ (#3)}} 

\theoremstyle{myremark}
\newtheorem{remark}{Remark}

\theoremstyle{definition}
\newtheorem{defn}{Definition}
\newtheorem{property}{Property}

\crefmultiformat{prop}{Propositions~#2#1#3}{ and~#2#1#3}{, #2#1#3}{ and~#2#1#3}
\crefmultiformat{theorem}{Theorems~#2#1#3}{ and~#2#1#3}{, #2#1#3}{ and~#2#1#3}
\crefmultiformat{lemma}{Lemmas~#2#1#3}{ and~#2#1#3}{, #2#1#3}{ and~#2#1#3}
\crefmultiformat{coro}{Corollaries~#2#1#3}{ and~#2#1#3}{, #2#1#3}{ and~#2#1#3}
\crefmultiformat{problem}{Problems~#2#1#3}{ and~#2#1#3}{, #2#1#3}{ and~#2#1#3}
\crefmultiformat{assump}{Assumptions~#2#1#3}{ and~#2#1#3}{, #2#1#3}{ and~#2#1#3}

\crefmultiformat{defn}{Definitions~#2#1#3}{ and~#2#1#3}{, #2#1#3}{ and~#2#1#3}
\crefrangeformat{defn}{Definitions.~#3#1#4,#5#2#6}

\crefmultiformat{property}{Properties~#2#1#3}{ and~#2#1#3}{, #2#1#3}{ and~#2#1#3}
\crefrangeformat{property}{Properties~#3#1#4,#5#2#6}

\newcommand{\parheading}[1]{{\noindent{\textit{#1}}}}

\DeclareMathOperator*{\argmin}{arg\,min}

\newcommand{\ob}{\ensuremath{\mathrm{ob}}}
\newcommand{\mor}{\ensuremath{\mathrm{Mor}}}
\newcommand{\id}{\ensuremath{\mathrm{id}}}

\newcommand{\bu}{\ensuremath{\mathbf{u}}}

\newcommand{\bA}{\ensuremath{\mathbf{A}}}

\newcommand{\cC}{\ensuremath{\mathcal{C}}}
\newcommand{\cD}{\ensuremath{\mathcal{D}}}

\newcommand{\cM}{\ensuremath{\mathcal{M}}}
\newcommand{\cN}{\ensuremath{\mathcal{N}}}
\newcommand{\cP}{\ensuremath{\mathcal{P}}}
\newcommand{\cS}{\ensuremath{\mathcal{S}}}
\newcommand{\cV}{\ensuremath{\mathcal{V}}}
\newcommand{\cW}{\ensuremath{\mathcal{W}}}

\newcommand{\sC}{\ensuremath{\mathsf{C}}}
\newcommand{\sD}{\ensuremath{\mathsf{D}}}
\newcommand{\sSet}{\ensuremath{\mathsf{Set}}}
\newcommand{\sM}{\ensuremath{\mathsf{M}}}
\newcommand{\sN}{\ensuremath{\mathsf{N}\mathcal{S}}}

\newcommand{\bbR}{\ensuremath{\mathbb{R}}}
\newcommand{\bbRp}{\ensuremath{\mathbb{R}^+}}
\newcommand{\bbZ}{\ensuremath{\mathbb{Z}}}

\begin{document}

\maketitle

\begin{abstract}
  With the increasing popularity of graph-based methods for dimensionality reduction and representation learning, node embedding functions have become important objects of study in the literature.
  In this paper, we take an axiomatic approach to understanding node embedding methods, first stating three properties for embedding dissimilarity networks, then proving that all three cannot be satisfied simultaneously by \emph{any} node embedding method.
  Similar to existing results on the impossibility of clustering under certain axiomatic assumptions, this points to fundamental difficulties inherent to node embedding tasks.
  Once these difficulties are identified, we then relax these axioms to allow for certain node embedding methods to be admissible in our framework.
\end{abstract}

\section{Introduction}
\label{sec:intro}

Graph-structured data is pervasive in the natural and social sciences, not only as a direct model for data, but as a useful abstraction for understanding complex relational systems.
This ubiquity has driven recent developments in graph representation learning~\cite{Hamilton:2020}, seeking to extract useful representations of graphs for classification and inference tasks.
In particular, node embedding methods embed the nodes of a graph in a low-dimensional space, in ways that encode the structural role of each node in a graph.
Indeed, the goal of node embedding methods is to represent the nodes of a graph in a space where ``similar'' nodes are close together, and ``different'' nodes are far apart~\cite{Hamilton:2020}.
Of course, the notion of two nodes being similar is subject to the practitioner's interpretation of the data, and the space in which the nodes are represented imposes geometric structure on the notion of proximity.

Inspired by the seminal paper of \citeauthor{Kleinberg:2002}~\cite{Kleinberg:2002}, we take an \emph{axiomatic view} of node embedding methods.
Specifically, we state three basic axioms that node embedding methods ought to satisfy, and then prove that no node embedding method can satisfy all three.
That is to say, a practitioner applying a node embedding method must sacrifice \emph{at least one} of these properties.
Following this, we present relaxations of these axioms that are not contradictory, allowing for the construction of node embedding algorithms that satisfy these relaxed properties simultaneously.

\section{Preliminaries}
\label{sec:prelim}

A typical type of data from which graphs, or networks, are constructed is a finite point cloud in a metric space.
Finding the triangle inequality unnecessary, we relax the metric condition to yield a \emph{dissimilarity network}.
To construct a dissimilarity network, begin with a finite set of nodes $\cV$.
For convenience, we will typically identify $\cV$ with the first $|\cV|$ natural numbers, saying $\cV=\{1,2,\ldots,|\cV|\}$.
A dissimilarity network, then, is a finite set of nodes $\cV$ coupled with a symmetric \emph{dissimilarity function} $d:\cV\times\cV\to\bbRp$, with the condition that $d(i,j)=0$ if and only if $i=j$.
Notice that metrics are dissimilarity functions, so that a finite metric space can be viewed as a dissimilarity network.

More generally, we speak of a \emph{dissimilarity space}, which is a set $\cS$ coupled with a dissimilarity function $\rho$ on $\cS$ (so that $\rho:\cS\times\cS\to\bbRp$).
If we relax the identity of indiscernibles, so that $x=y$ implies $\rho(x,y)=0$, but $\rho(x,y)=0$ does not imply $x=y$,
as well as allow for $\rho$ to take values in $\bbRp\cup\{+\infty\}$, we call $\rho$ an \emph{pseudodissimilarity function}.
A \emph{pseudodissimilarity space} is then a set coupled with a pseudodissimilarity function on it.
Let $\cM$ be the set of all finite pseudodissimilarity spaces.
This allows us to endow a dissimilarity network with features on the nodes.
Let $(\cS,\rho)$ be a dissimilarity space.
An \emph{$\cS$-featured dissimilarity network} is a dissimilarity network $(\cV,d)$ and a \emph{feature map} $F:\cV\to\cS$, denoted $(\cV,d,F)$.
The set of all such objects is denoted $\cN(\cS)$.
We will speak of \emph{featured dissimilarity networks} when $\cS$ is understood from context.

Graph representation learning seeks to represent networks in tractable spaces, typically Euclidean space or on a low-dimensional manifold~\cite{Hamilton:2020}.
This motivates our object of study, which is the class of \emph{node embedding functions}.
A node embedding function is a map $\xi:\cN(\cS)\to\cM$ with the condition that for any featured dissimilarity network $N=(\cV,d,F)$, the corresponding dissimilarity space has the same underlying set: $\xi(N)=(\cV,\phi)$, for some dissimilarity function $\phi$ on $\cV$.
In the context of embedding nodes into some predefined space, such as $\bbR^n$, this is equivalent to restricting the dissimilarity function on that space to the embedded points of $\cV$.
For instance, if the embedding function maps the nodes into $\bbR^n$ with the usual metric, then $\phi$ is determined by the Euclidean distances between the embedded nodes.

\section{Main Result}
\label{sec:main}

A typical goal in node embedding tasks is to embed nodes in a way that \emph{preserves} and \emph{reflects} the original network structure.
We begin by defining three properties of node embedding functions that capture these goals, namely being \emph{self-contained}, \emph{consistent}, and \emph{graph-aware}.
Throughout, let $(\cS,\rho)$ be a dissimilarity space, and consider node embedding functions $\xi:\cN(\cS)\to\cM$.

In many graph representation learning tasks, we seek embeddings that are invariant/equivariant to the labeling of nodes.
In particular, the embedding of nodes via some node embedding function should only be dependent on their pairwise dissimilarities and features.
We state this for dissimilarity networks with $2$ nodes as follows.

\begin{property}[Self-containedness]\label{prpt:contained}
  Let $N_2=(\{i,j\},d,F)$ be a featured dissimilarity network on $2$ nodes, and put $\alpha=d(i,j)$, so that $\alpha>0$.
  For some node embedding function $\xi$, put $(\{i,j\},\phi)=\xi(N_2)$.
  Suppose there exists a function $g$ such that for all such networks,
  \begin{equation}
    \phi(i,j) = g(F(i),F(j);\alpha).
  \end{equation}
  Under these conditions, we say that $\xi$ is \emph{self-contained}.
\end{property}

Self-containedness is perhaps the most self-evident property we propose.
In particular, it requires node embedding functions to depend on the network structure alone, rather than other external information.
Stated in the most basic case for networks on two nodes, it is a primitive form of permutation invariance.
That is, when embedding a two-node network, the dissimilarity in the embedded space should only depend on the dissimilarity in the original network and the features of both nodes.
Another useful property of a node embedding function is for it to preserve proximity between nodes: that is, if two nodes $i,j$ are similar, their embeddings should be similar as well.
So, if we shrink the dissimilarity between nodes, their dissimilarity in the embedding space should decrease accordingly.
We state this by first defining the notion of a ``contractive and dissimilarity non-increasing'' map from one network to another.

\begin{defn}\label{defn:codni}
  Consider two featured dissimilarity networks $N_1=(\cV_1,d_1,F_1), N_2=(\cV_2,d_2,F_2)$ with possibly different numbers of nodes.
  A map $\Psi:\cV_1\to\cV_2$ is said to be \emph{contractive and dissimilarity non-increasing} (CoDNI) if, for all $i,j\in\cV_1$,
  \begin{gather}
    \rho(F_1(i),F_1(j))\geq\rho(F_2(\Psi(i)),F_2(\Psi(j))) \\
    d_1(i,j)\geq d_2(\Psi(i),\Psi(j)).
  \end{gather}
\end{defn}

Notice that the image nodes under a CoDNI map are closer in feature space and have smaller network dissimilarities than their corresponding preimages.
Thus, given a CoDNI map from one network to another, a reasonable property of a node embedding is to reduce the distance between nodes in the embedding space.
Alternatively, this property requires node embeddings to preserve \emph{local} dissimilarity information.

\begin{property}[Consistency]\label{prpt:consistent}
  Consider two featured dissimilarity networks $N_1=(\cV_1,d_1,F_1), N_2=(\cV_2,d_2,F_2)$ with possibly different numbers of nodes.
  For some node embedding function $\xi$, put $(\cV_1,\phi_1)=\xi(N_1)$ and $(\cV_2,\phi_2)=\xi(N_2)$.
  We say that $\xi$ is \emph{consistent} if, for all $i,j\in\cV_1$ and CoDNI maps $\Psi:\cV_1\to\cV_2$,
  \begin{equation}
    \phi_1(i,j)\geq\phi_2(\Psi(i),\Psi(j)).
  \end{equation}
\end{property}

Finally, we wish for the node embedding to reflect the \emph{global} dissimilarity structure of the underlying network as well.
We express this in terms of the dissimilarity between the embedding of a given pair of nodes, where we increase all other pairwise dissimilarities in the network in order to change the embedded dissimilarity between the two nodes.

\begin{property}[Graph-awareness]\label{prpt:aware}
  Consider a featured dissimilarity network $N=(\cV,d,F)$ on at least $3$ nodes, and consider an arbitrary pair of nodes $i,j\in\cV$ such that $i\neq j$.
  Let $\cD_{ij}$ be the set of dissimilarity functions $d'$ such that $d'(i,j)=d(i,j)$, and for all $k,\ell\in\cV$ it holds that $d'(k,\ell)\geq d(k,\ell)$.
  For some node embedding function $\xi$, if for all such networks $N$ there exists some $d'\in\cD_{ij}$ such that the embeddings $(\cV,\phi)=\xi(N), (\cV,\phi')=\xi((\cV,d',F))$ satisfy $\phi(i,j)\neq\phi'(i,j)$, then we say that $\xi$ is \emph{graph-aware}.
\end{property}

With these definitions in place, we now state our impossibility result.

\begin{theorem}\label{thm:impossible}
  There does not exist a node embedding function that is simultaneously \hyperref[prpt:contained]{self-contained}, \hyperref[prpt:consistent]{consistent}, and \hyperref[prpt:aware]{graph-aware}.
\end{theorem}

We prove this result in \cref{sec:proof}.
\Cref{thm:impossible} reflects a fundamental difficulty inherent to the node embedding task.
One wishes for a node embedding method to preserve \emph{local} structures, while still being sensitive to \emph{global} structure as well.
However, preserving both of these in a way consistent with our axioms is impossible: intuitively, one cannot suitably capture both types of information within a single embedding.

\begin{remark}
  Following the categorical view of clustering algorithms taken by~\cite{Carlsson:2013}, 
  studying node embedding functions as functors between categories naturally yields self-containedness and consistency as desirable properties, with graph-awareness defined as before in order to demand sensitivity to global network structure.
  Let $\sN$ be the category of $\cS$-featured dissimilarity networks with CoDNI maps as morphisms, 
  $\sM$ be the category of pseudodissimilarity spaces as objects with non-expansive maps as morphisms, 
  and $\sSet$ be the category of sets with functions between them as morphisms.
  We rephrase \cref{thm:impossible} in this framework, noting that the properties of self-containedness and consistency are implied by functoriality:
  \begin{reptheorem}{thm:impossible}[Categorical Version]
    There exists no \hyperref[prpt:aware]{graph-aware} functor $\xi:\sN\to\sM$.
  \end{reptheorem}
  We discuss the definitions of these objects and the relationship between both statements of \cref{thm:impossible} in \cref{app:categorical}.
\end{remark}

\section{Examples}
\label{sec:example}

In light of \cref{thm:impossible}, any node embedding method that satisfies two of the three properties of self-containedness, consistency, and graph-awareness must fail to satisfy the third.
We illustrate this with a few examples of node embedding functions previously considered in the literature~\cite{Kleinberg:2002,Carlsson:2013}, which we examine through the lens of our proposed axioms.

\subsection{Single-Linkage Clustering}
\label{sec:example:single-linkage}


We consider the single-linkage clustering procedure, which embeds dissimilarity networks in ultrametric spaces~\cite{Johnson:1967}.
Indeed, dendrograms represent ultrametric spaces, where the distance between two points is the point at which they are merged into the same cluster.

We first discuss a useful construction for defining the embedding distance of the single-linkage procedure.
\begin{defn}
  For a dissimilarity network $N=(\cV,d,F)$, and two nodes $i,j\in\cV$, a \emph{path from $i$ to $j$} is a finite sequence $P=[x_1,x_2,\ldots,x_{L-1},x_L]$ of nodes in $\cV$ such that $x_1=i, x_L=j$.
  The set of all paths from $i$ to $j$ is denoted $\cP_{ij}$.
  We refer to the $\ell\mathrm{th}$ element of a path $P$ by $P(\ell)$.
  The \emph{length} of a path is the number of elements in the path, denoted $L(P)$.
\end{defn}
For a dissimilarity network $N=(\cV,d,F)$, the single-linkage procedure $(\cV,\phi)=\xi(N)$ is such that the embedding distance between two nodes $i,j\in\cV$ is given by
\begin{equation}\label{eq:single-linkage-distance}
  \phi(i,j) = \min_{P\in\cP_{ij}}\max_{1\leq\ell <L(P)}d(P(\ell),P(\ell+1)).
\end{equation}
Notice that this embedding disregards node features.
Let us examine this embedding distance with respect to each of our proposed properties.

\begin{prop}\label{prop:single-linkage-fails}
  The single-linkage procedure is \hyperref[prpt:contained]{self-contained} and  \hyperref[prpt:consistent]{consistent}, but not \hyperref[prpt:aware]{graph-aware}.
\end{prop}

\begin{proof}
Although in light of \Cref{thm:impossible} it is sufficient to verify that self-containedness and consistency hold to prove \cref{prop:single-linkage-fails}, we discuss all three properties for the sake of completeness.

\parheading{Self-containedness.}
Let $N_2=(\{i,j\},d,F)$ be a featured dissimilarity network on $2$ nodes, and let $(\{i,j\},\phi)=\xi(N_2)$.
Observe that $P=[i,j]$ is the only path from $i$ to $j$ (disregarding paths with repeated entries), so that by \eqref{eq:single-linkage-distance} we have
\begin{equation}
  \phi(i,j) = d(i,j).
\end{equation}
This is consistent with \cref{prpt:contained}, where $g(F(i),F(j);\alpha)=\alpha$, as desired.

\parheading{Consistency.}
Let two featured dissimilarity networks $N_1=(\cV_1,d_1,F_1), N_2=(\cV_2,d_2,F_2)$ be given such that there exists a CoDNI map $\Psi:\cV_1\to\cV_2$.
Let $(\cV_1,\phi_1)=\xi(N_1), (\cV_2,\phi_2)=\xi(N_2)$.
Let nodes $i,j\in\cV_1$ be given, and consider a path $P$ from $i$ to $j$ (in $\cV_1$).
Construct the path $\widehat{P}$ from $\Psi(i)$ to $\Psi(j)$ by applying $\Psi$ to each element in $P$.
Since $\Psi$ is CoDNI, we have that $d_1(P(\ell),P(\ell+1))\geq d_2(\widehat{P}(\ell), \widehat{P}(\ell+1))$ for all $1\leq\ell < L(P)=L(\widehat{P})$.
It follows that
\begin{equation}
  \max_{1\leq\ell < L(P)}
  d_1(P(\ell),P(\ell+1))
  \geq
  \max_{1\leq\ell < L(\widehat{P})}
  d_2(\widehat{P}(\ell),\widehat{P}(\ell+1)).
\end{equation}
Examining \eqref{eq:single-linkage-distance}, this implies that
\begin{equation}
  \phi_1(i,j)\geq\phi_2(\Psi(i),\Psi(j)),
\end{equation}
which is consistent with \cref{prpt:consistent}, as desired.

\parheading{Graph-awareness.}
Since the other properties hold, \cref{thm:impossible} implies that the single-linkage procedure is not graph-aware.
To verify this, let a featured dissimilarity network $N=(\cV,d,F)$ on at least $3$ nodes be given, and take $i,j\in\cV$ to be the pair of nodes with minimum dissimilarity.
That is,
\begin{equation}
  i,j\in\argmin_{i',j'\in\cV:i'\neq j'}d(i',j').
\end{equation}
One can check that the path $P=[i,j]$ always attains the minimum distance in \eqref{eq:single-linkage-distance}.
Moreover, for any $d'\in\cD_{ij}$ as described in \cref{prpt:aware}, the expression \eqref{eq:single-linkage-distance} increases with respect to the distance $d'$.
Since $d(i,j)=d'(i,j)$ by definition, the path $P=[i,j]$ still attains the minimum distance in \eqref{eq:single-linkage-distance} for any $d'\in\cD_{ij}$, and said distance is the same as that induced by $d$.
Thus, the single-linkage procedure is not graph-aware.
\end{proof}

\begin{remark}
  The single-linkage procedure for embedding nodes in an ultrametric space is a special case of a metric projection of a dissimilarity network, as described in \cite{Segarra:Metric,Segarra:2020}.
  Similar arguments show that such metric projections satisfy self-containedness and consistency, but fail to satisfy graph-awareness.
\end{remark}

\subsection{Motif-Based Embedding}
\label{sec:example:motif}

Single-linkage clustering relies strictly on pairwise distances between nodes, leading to undesirable properties such as sensitivity to noise, chaining~\cite{Lance:1967}, and failure to incorporate global structure, as reflected by the fact that the single-linkage procedure is not graph-aware.
To remedy this, we consider another ultrametric embedding based on triangular motifs, which we refer to as the \emph{triangle-linkage procedure}, based on the clustering scheme discussed in~\cite[Section~6.7]{Carlsson:2013}.

For a dissimilarity network $N=(\cV,d,F)$, the triangle-linkage procedure with respect to $L$, denoted $(\cV,\phi)=\xi_T(N)$, is such that the embedding dissimilarity between any distinct nodes $i,j\in\cV$ is $+\infty$ when $|\cV|<3$.
Otherwise, we have
\begin{equation}
  \phi(i,j) = \min\{\epsilon>0:\exists k\in\cV, i\neq k, j\neq k, \max\{d(i,j),d(i,k),d(j,k)\}\leq\epsilon\}.
\end{equation}
That is, two nodes $i,j$ have an embedded dissimilarity of at most $\epsilon$ if there is a third node $k$ such that the triple $i,j,k$ forms a triangle where each side has dissimilarity less than or equal to $\epsilon$.
One could easily modify this definition for other motifs, such as small cliques, loops, and others.
As before, we now examine this embedding distance with respect to our proposed properties.

\begin{prop}\label{prop:motif-linkage-fails}
  The triangle-linkage procedure is \hyperref[prpt:contained]{self-contained} and  \hyperref[prpt:aware]{graph-aware}, but not \hyperref[prpt:consistent]{consistent}.
\end{prop}

\begin{proof}
As before, we verify this statement for all three properties.

\parheading{Self-containedness.}
Since the node set of a $2$-node network $N=(\{i,j\},d,F)$ has cardinality fewer than $3$ nodes, the embedding distance between its constituent nodes under the triangle-linkage procedure is always $+\infty$: this satisfies self-containedness where $g(F(i),F(j);\alpha)=+\infty$.

\parheading{Graph-awareness.}
Let a featured dissimilarity network $N=(\cV,d,F)$ be given such that $|\cV|\geq 3$.
Let $(\cV,\phi)=\xi_T(N)$ be the triangle-linkage embedding of $N$, and pick distinct nodes $i,j\in\cV$.
Put $\delta=\phi(i,j)$, and take $\cW\subseteq\cV$ to be the set of all nodes $k\in\cV$ such that $d(i,k)\leq\delta, d(j,k)\leq\delta$.
Notice that $\cW$ is nonempty.
Define the dissimilarity function $d'\in\cD_{ij}$ such that for all $k\in\cW$, we have $d'(i,k)=d'(j,k)=2\delta$.
Letting $(\cV,\phi')=\xi_T((\cV,d',F))$, one can verify that $\phi'(i,j)>\delta$, so that the triangle-linkage procedure is graph-aware, as desired.

\parheading{Consistency.}
Since the other properties hold, \cref{thm:impossible} implies that the triangle-linkage procedure is not consistent.
To verify this, consider two featured dissimilarity networks, $N_1=(\{i_1,j_1,k_1\},d_1,F_1), N_2=(\{i_2,j_2,k_2\},d_2,F_2)$, where $d_1,d_2$ are defined so that
\begin{align*}
  d_1(i_1,j_1)&=\delta & d_2(i_2,j_2)&=\delta/2 \\
  d_1(i_1,k_1)&=\delta & d_2(i_2,k_2)&=2\delta  \\
  d_1(j_1,k_1)&=\delta & d_2(j_2,k_2)&=2\delta,
\end{align*}
for some $\delta>0$.
Letting $(\{i_1,j_1,k_1\},\phi_1)=\xi_T(N_1), (\{i_2,j_2,k_2\},\phi_2)=\xi_T(N_2)$, we see that $\phi_1(i_1,j_1)=\delta$ and $\phi_2(i_2,j_2)=2\delta$.
Define the CoDNI map $\Psi:\{i_1,j_1,k_1\}\to\{i_2,j_2,k_2\}$ so that
\begin{equation}
  \Psi
  \begin{cases}
    i_1\mapsto i_2 \\
    j_1\mapsto j_2 \\
    k_1\mapsto j_2.
  \end{cases}
\end{equation}
Despite $\Psi$ being a CoDNI map, we have $\phi_1(i_1,j_1)<\phi_2(\Psi(i_1),\Psi(j_1))$, so that the triangle-linkage procedure is not consistent.
\end{proof}

\subsection{Spectral Embedding}

While the above single-linkage and triangle-linkage procedures are combinatorial methods for embedding dissimilarity networks, \emph{spectral methods} are also a common approach~\cite{belkin2003laplacian}.
In particular, one constructs a graph given pairwise distances between nodes, and weights the edges of that graph based on some kernel.
Then, spectral methods are used to embed the nodes in Euclidean space.
We consider a very simple procedure inspired by this, where the leading eigenvector of the constructed adjacency matrix is used to embed nodes in $\bbR$, taking the standard Euclidean metric as the embedding distance.

Let $\kappa:\bbRp\times\cS\times\cS\to\bbRp$ be a continuous (with respect to the real-valued argument) function mapping dissimilarities and pairs of features to the nonnegative real numbers such that $\kappa(0;s_1,s_1)=1, \kappa(x;s_1,s_2)=\kappa(x;s_2,s_1)>0$ for all $x\in\bbRp, s_1,s_2\in\cS$ and is monotonically decreasing with respect to $x$ and $\rho(s_1,s_2)$, where $\kappa(x;s_1,s_2)\to 0$ as $x\to\infty$.
For some $\cV=\{1,2,\ldots,n\}$, let $N=(\cV,d,F)$ be an $\cS$-featured dissimilarity network, and construct an adjacency matrix $\bA\in\bbR^{n\times n}$ where $A_{ij}=\kappa(d(i,j);F(i),F(j))$.
Let $\bu$ be a normalized leading eigenvector of $\bA$, with corresponding eigenvalue $\lambda>0$, as guaranteed by the Perron-Frobenius theorem.
Here, $\bu$ is the \emph{eigenvector centrality vector} of $\bA$~\cite{Bonacich:1987}.
Denoting the node embedding of $N$ as $(\cV,\phi)=\xi_C(N)$, the embedding dissimilarity $\phi$ is defined for each pair of nodes $i,j\in\cV$ as
\begin{equation}
    \phi(i,j) = \sqrt{\lambda}|u_i-u_j|.
\end{equation}
We refer to this embedding as the \emph{eigenvector centrality embedding}.
Its behavior with respect to \cref{thm:impossible} is characterized as follows:
\begin{prop}\label{prop:eigvec-fails}
  The eigenvector centrality embedding procedure is \hyperref[prpt:contained]{self-contained} and  \hyperref[prpt:aware]{graph-aware}, but not \hyperref[prpt:consistent]{consistent}.
\end{prop}
We leave the proof to \cref{app:spectral}.
This is not too surprising, though, since eigenvectors of matrices are \emph{global} quantities of the matrix structure, so graph-awareness can be readily demonstrated.
Moreover, the eigenvector centrality is not sensitive to permutations of nodes, so that self-containedness holds as well.

\section{Kleinberg's Impossibility Theorem}
\label{sec:clustering}

\citeauthor{Kleinberg:2002}~\cite{Kleinberg:2002} established three axioms for \emph{clustering functions}, which take as input a set of points $\cV$ coupled with a dissimilarity function $d$, and yields a partition of $\cV$.
By appropriately defining the notions of \emph{scale-invariance,} \emph{richness,} and \emph{consistency,}%
\footnote{Consistency as defined by Kleinberg is distinct from \cref{prpt:consistent}, although related in studying maps that are contractive in some sense.}
he proved an impossibility result for clustering functions, namely that no clustering function satisfies these three properties simultaneously.

Our results draw great inspiration from this approach, and indeed apply to clustering functions as special types of node embedding functions.
To see this, let a partition of $\cV$ be given, so that $\cV=\bigcup_{j=1}^L \cV_j$, where $\cV_j\cap\cV_k=\emptyset$ for $j\neq k$.
For each $v\in\cV$, there exists a unique integer $j$ such that $v\in\cV_j$.
If two (not necessarily distinct) nodes $v,w\in\cV$ are contained in the same partition $\cV_j$, we say that $v\sim w$.
For some fixed value $\epsilon>0$, define a pseudodissimilarity function on $\cV$:
\begin{equation}
    \phi(v,w)=
    \begin{cases}
    0 & v\sim w \\
    \epsilon & \text{otherwise}.
    \end{cases}
\end{equation}
One can see that $\phi$ encodes the partition $\{\cV_j\}_{j=1}^L$, where two distinct nodes are in the same partition if and only if they have dissimilarity $0$, and are in differing partitions if and only if they have dissimilarity $\epsilon$.
One can see that there is an injection from the set of partitions of $\cV$ and all such pseudodissimilarity functions on $\cV$.
Thus, our axioms yield a new impossibility result for clustering, as a particular type of node embedding.

\section{Proof of \Cref{thm:impossible}}
\label{sec:proof}

Let $(\cS,\rho)$ be a dissimilarity space.
Suppose, for the sake of contradiction, that $\xi:\cN(\cS)\to\cM$ is a node embedding function that is self-contained, consistent, and graph-aware.
For some $n\geq 3$, let $\cV=\{1,2,\ldots,n\}$ and some $F:\cV\to\cS$ be given.
Pick $i,j\in\cV$ that achieves minimum dissimilarity in $\cS$, so that
\begin{equation}\label{eq:proof:impossible:minimize}
  i,j\in\argmin_{i',j'\in\cV:i'\neq j'}\rho(F(i),F(j)),
\end{equation}
assuming without loss of generality that $i<j$.
Fix some $\delta>0$ and define a dissimilarity function $d$ on $\cV$ so that $d(k,\ell)=\delta$ for all $k\neq\ell$, taking value $0$ otherwise.
This yields a featured dissimilarity network $N=(\cV,d,F)$.

Since $\xi$ is graph-aware, choose a dissimilarity function $d_1\in\cD_{ij}$ so that if $N_1=(\cV,d_1,F)$ and $(\cV,\phi)=\xi(N), (\cV,\phi_1)=\xi(N_1)$, we have $\phi(i,j)\neq\phi_1(i,j)$.

Additionally, put $N_2=(\{i,j\},d\big|_{\{i,j\}^2},F\big|_{\{i,j\}})$.
That is to say, $N_2$ is the restriction of $N$ to the nodes $i,j$.
Put $(\{i,j\},\phi_2)=\xi(N_2)$, so that by self-containment we have
\begin{equation}\label{eq:proof:impossible:contained1}
  \phi_2(i,j)=g(F(i),F(j);d(i,j))=g(F(i),F(j);\delta),
\end{equation}
where $g$ is the function whose existence is guaranteed by \cref{prpt:contained}.

Consider the map $\Psi:\{i,j\}\hookrightarrow\cV$ defined via inclusion, so that $\Psi(i)=i, \Psi(j)=j$.
Notice that $\Psi$ is a CoDNI map from $N_2$ to $N_1$.
By consistency of $\xi$, then, we have
\begin{equation}\label{eq:proof:impossible:bound1}
  \phi_1(i,j)\leq\phi_2(i,j)=g(F(i),F(j);\delta),
\end{equation}
where the equality follows from~\eqref{eq:proof:impossible:contained1}.
Similarly, consider the map $\Psi':\cV\to\{i,j\}$ where $\Psi'(k)=i$ for $k\leq i$ and $\Psi'(k)=j$ for $k>i$ (in particular, $\Psi'(j)=j$).
For any $k,\ell\in\cV$ such that $k\neq\ell$, we have
\begin{equation}\label{eq:proof:impossible:featurecontraction}
  \rho(F(k),F(\ell))
  \overset{(a)}{\geq}\rho(F(i),F(j))
  \overset{(b)}{\geq}\rho(F(\Psi'(k)),F(\Psi'(\ell))),
\end{equation}
where $(a)$ follows from our choice of $i,j$ in \eqref{eq:proof:impossible:minimize}, and $(b)$ follows from the fact that $\rho(F(\Psi'(k)),F(\Psi'(\ell)))$ is either equal to $\rho(F(i),F(j))$ or zero.
Moreover, again for $k\neq\ell$,
\begin{equation}\label{eq:proof:impossible:distancecontraction}
  d_1(k,\ell)
  \overset{(a)}{\geq}d(k,\ell)
  =\delta
  \geq d_2(\Psi'(k),\Psi'(\ell)),
\end{equation}
where $(a)$ follows from our choice of $d_1\in\cD_{ij}$.

Examining \eqref{eq:proof:impossible:featurecontraction} and \eqref{eq:proof:impossible:distancecontraction}, we see that $\Psi'$ is a CoDNI map from $N_1$ to $N_2$.
By consistency of $\xi$, we have
\begin{equation}\label{eq:proof:impossible:bound2}
  \phi_1(i,j)
  \geq\phi_2(i,j)
  \overset{(a)}{=}g(F(i),F(j);\delta),
\end{equation}
where $(a)$ invokes \eqref{eq:proof:impossible:contained1}.

Combining the inequalities \eqref{eq:proof:impossible:bound1} and \eqref{eq:proof:impossible:bound2}, we have
\begin{equation}\label{eq:proof:impossible:bound3}
  \phi_1(i,j)=g(F(i),F(j);\delta).
\end{equation}
The same argument can be applied to $N$ (that is, constructing $\Psi,\Psi'$ between $N$ and $N_2$), yielding
\begin{equation}\label{eq:proof:impossible:bound4}
  \phi(i,j)=g(F(i),F(j);\delta).
\end{equation}
Combining \eqref{eq:proof:impossible:bound3} and \eqref{eq:proof:impossible:bound4}, we see that
\begin{equation}\label{eq:proof:contradiction}
  \phi(i,j)=\phi_1(i,j),
\end{equation}
contradicting our choice of $d_1\in\cD_{ij}$ so that $\phi(i,j)\neq\phi_1(i,j)$, as desired.\hfill$\square$

\begin{remark}
In the proof of \cref{thm:impossible}, the networks $N,N_2$ have dissimilarity functions that satisfy the triangle inequality, and are thus metrics.
If one restricts the domain of node embedding functions to dissimilarity networks with metric dissimilarities, our impossibility result still holds.
\end{remark}

\section{Relaxations}
\label{sec:relax}

Although all node embedding methods must violate at least one of the three properties we put forth due to \cref{thm:impossible}, it is worthwhile to study relaxations of these properties that yield admissible node embeddings.
In this section, we consider the implications of relaxing graph-awareness and consistency.

\subsection{Weak Graph-Awareness}
\label{sec:relax:aware}

Recall the definition of graph-awareness, where the embedding dissimilarity of two nodes $i,j$ must be sensitive to sufficient increases in the values taken by the network dissimilarity function, described by the set $\cD_{ij}$.
Here, we consider a relaxation of this property, where we allow for any perturbation of the network dissimilarity that preserves the dissimilarity between nodes $i,j$.
We refer to this property as \emph{weak graph-awareness}:
\begin{property}[Weak graph-awareness]\label{prpt:weakly-aware}
  Consider a featured dissimilarity network $N=(\cV,d,F)$ on at least $3$ nodes, and consider an arbitrary pair of nodes $i,j\in\cV$ such that $i\neq j$.
  Let $\cC_{ij}$ be the set of distance functions $d'$ such that $d'(i,j)=d(i,j)$.
  For some node embedding function $\xi$, if for all such networks there exists some $d'\in\cC_{ij}$ such that the embeddings $(\cV,\phi)=\xi(N), (\cV,\phi')=\xi((\cV,d',F))$ satisfy $\phi(i,j)\neq\phi'(i,j)$, then we say that $\xi$ is \emph{weakly graph-aware}.
\end{property}
Observe that all graph-aware node embedding functions are weakly graph-aware, since $\cD_{ij}\subseteq\cC_{ij}$.
Looking back to our analysis of the single-linkage procedure, this relaxation is sufficient for the single-linkage embedding to satisfy all three axioms.

\begin{prop}
  The single-linkage procedure is \hyperref[prpt:contained]{self-contained}, \hyperref[prpt:consistent]{consistent}, and \hyperref[prpt:weakly-aware]{weakly graph-aware}.
\end{prop}
\begin{proof}
  By \Cref{prop:single-linkage-fails}, the single-linkage procedure is self-contained and consistent, leaving weak graph-awareness to be shown.

  Let a featured dissimilarity network $N=(\cV,d,F)$ on at least $3$ nodes be given, and let $(\cV,\phi)=\xi(N)$ be the single-linkage node embedding.
  Take nodes $i,j\in\cV$ such that $i\neq j$, and put $\alpha=\phi(i,j)$, noting that $\alpha\neq 0$.
  Pick a third node $k\in\cV$, and choose any dissimilarity function $d':\bbZ\times\bbZ\to\bbRp$ such that $d'(i,j)=d(i,j), d'(i,k)=d'(j,k)=\alpha/2$.
  Note that $d'\in\cC_{ij}$.
  Let $(\cV,\phi')=\xi((\cV,d',F))$, and observe that $\phi'(i,j)\leq\alpha/2$.
  Thus, $\phi'(i,j)\neq\phi(i,j)$, as desired.
\end{proof}

The relaxation of graph-awareness to weak graph-awareness can be interpreted as making the node embedding function less sensitive to the global structure of the network.
In doing so, we allow for arbitrary perturbations of the network structure, rather than ones that only increase dissimilarity.

\subsection{Injective Consistency}
\label{sec:relax:consistent}

In the proof of \cref{thm:impossible}, a key step is forming a CoDNI map from a network of many nodes to a network of two nodes.
By the pigeonhole principle, a function mapping a set to another set with smaller cardinality is not injective.
This motivates a stronger requirement on the map $\Psi$ in our definition of consistency, yielding a property that we call \emph{injective consistency}:
\begin{property}[Injective consistency]\label{prpt:inj-consistent}
  Consider two featured dissimilarity networks $N_1=(\cV_1,d_1,F_1), N_2=(\cV_2,d_2,F_2)$ with possibly different numbers of nodes.
  For some node embedding function $\xi$, let $(\cV_1,\phi_1)=\xi(N_1), (\cV_2,\phi_2)=\xi(N_2)$.
  We say that $\xi$ is \emph{injectively consistent} if, for all $i,j\in\cV_1$ and injective CoDNI maps $\Psi:\cV_1\to\cV_2$,
  \begin{equation}
    \phi_1(i,j)\geq\phi_2(\Psi(i),\Psi(j)).
  \end{equation}
\end{property}
Observe that all node embedding functions that are injectively consistent are also consistent, since injective consistency depends on a stronger hypothesis (namely, the CoDNI map $\Psi$ being injective).
We now reconsider the triangle-linkage procedure given this relaxed notion of consistency.

\begin{prop}
  The triangle-linkage procedure is \hyperref[prpt:contained]{self-contained}, \hyperref[prpt:inj-consistent]{injectively consistent}, and \hyperref[prpt:aware]{graph-aware}.
\end{prop}
\begin{proof}
  By \Cref{prop:motif-linkage-fails}, the triangle-linkage procedure is self-contained and graph-aware, leaving injective consistency to be shown.

  Let two featured dissimilarity networks $N_1=(\cV_1,d_1,F_1), N_2=(\cV_2,d_2,F_2)$ be given, and suppose there exists an injective CoDNI map $\Psi:\cV_1\to\cV_2$.
  Put $(\cV_1,\phi_1)=\xi_T(N_1), (\cV_2,\phi_2)=\xi_T(N_2)$.
  For any distinct nodes $i,j\in\cV_1$, put $\delta=\phi_1(i,j)$, so that there exists $k\in\cV_1$ such that $k\neq i, k\neq j$, and $\max\{d_1(i,j),d_1(i,k),d_1(j,k)\}=\delta$.
  Since $\Psi$ is an injective CoDNI map, $\Psi(i),\Psi(j),\Psi(k)$ are distinct elements of $\cV_2$, and $\max\{d_2(\Psi(i),\Psi(j)),d_2(\Psi(i),\Psi(k)),d_2(\Psi(j),\Psi(k))\}\leq\delta$, so that $\phi_2(\Psi(i),\Psi(j))\leq\delta$.
  That is to say, the triangle-linkage procedure is injectively consistent, as desired.
\end{proof}

In the same way that weak graph-awareness makes the node embedding less sensitive to changes in the global network structure, relaxing consistency to injective consistency weakens the sensitivity of node embedding functions to local changes in the network structure.
This illustrates the fundamental tradeoff described by \cref{thm:impossible}: a node embedding function can either be sensitive to global structure or local structure, but not both.

\section{Related Work}
\label{sec:related}

Node embedding is a rich, highly active research topic, with existing works falling into several categories~\cite{zhang2018network, cai2018comprehensive, goyal2018graph}.
We briefly survey methods based on matrix factorization, word embeddings, and neural architectures.

\parheading{Matrix factorization.}
The first approach is based on matrix factorization, where some matrix $S$ characterizing node similarities is factorized as $X^{\top}X$ or $X^{\top}Y$.
Then, the columns of $X$ (and possibly $Y$) are used as node embeddings in Euclidean space.
Many such methods have been proposed that consider different similarity matrices $S$ and different approximation schemes.
Approaches range from classical methods such as spectral clustering~\cite{von2007tutorial}, locally linear embedding~\cite{roweis2000nonlinear}, Laplacian eigenmaps~\cite{belkin2003laplacian}, multidimensional scaling~\cite{cox2008multidimensional} and Isomap~\cite{tenenbaum2000global}, as well as modern approaches such as GraRep~\cite{grarep}, HOPE~\cite{hope}, and NetMF~\cite{Qiu:2018}.

\parheading{Word embedding.}
Distinct from methods based on matrix factorization, many node embedding methods are derivatives of word2vec~\cite{word2vec_1,word2vec_2}, which learns embeddings of words from a corpus of text.
The earliest such approach is DeepWalk~\cite{deepwalk}, which treats nodes as words and random walks as sentences, so that word2vec can be directly applied to generate node embeddings.
Variants of DeepWalk consider different neighborhood sampling strategies:
for example, the well-known node2vec algorithm~\cite{node2vec} employs biased second-order random walks.
It has been proven that many of these methods \emph{implicitly} factorize particular similarity matrices~\cite{word_mf, Qiu:2018}, providing a link to more classical approaches.

\parheading{Neural architectures.}
The third category is inspired by the recent success of geometric deep learning~\cite{Bronstein:2017}.
Many such works apply autoencoder architectures to construct meaningful node embeddings,
including DNGR~\cite{cao2016deep}, SDNE~\cite{wang2016structural}, and ARGA/ARVGA~\cite{pan2018adversarially}.
Moreover, instead of taking only the graph structure into consideration, approaches have also been proposed which are able to incorporate additional features on the nodes and edges~\cite{shervashidze2011weisfeiler, hamilton2017inductive, wang2017knowledge}. 

\parheading{Theoretical analysis of general node embeddings.}
Although various node embedding methods have been proposed, theoretical analysis regarding their properties and limitations independent of any particular algorithm is sparsely found in the literature.
\cite{Srinivasan:2019} studied the relationships between node embeddings in Euclidean space and so-called ``structural representations.'' 
A recent work~\cite{seshadhri2020impossibility} proves that graphs generated from low-dimensional embeddings (using dot products as a similarity measure) cannot be both sparse and have high triangle density, two hallmarks of real-word networks.
However, a follow up paper~\cite{chanpuriya2020node} shows that the impossibility results in~\cite{seshadhri2020impossibility} are a consequence of the specific model considered. 
Specifically, the analysis in~\cite{seshadhri2020impossibility} associates each node with a single embedding vector, yielding a matrix factorization of the form $X^{\top}X$, which is necessarily positive semidefinite (PSD).
The paper~\cite{chanpuriya2020node} relaxes this model, allowing each node to be associated to two embedding vectors, corresponding to the non-PSD factorization $X^{\top}Y$.
This relaxation allows low-dimensional embeddings to capture sparse graphs with high triangle density.

\parheading{Axiomatic approaches.}
In this work, we provide theoretical insights for node embedding using an axiomatic framework.
A similar approach was originally taken by \cite{Kleinberg:2002} in the context of clustering, which proposes three axioms -- namely scale invariance, richness, and consistency -- and shows that it is impossible for a clustering function to simultaneously satisfy all of them.
Indeed, a small cottage industry emerged from the work~\cite{Kleinberg:2002}, in which suitable relaxations of these clustering axioms are formed to yield possibility and uniqueness results~\cite{Ben:2008,Carlsson:2013,Cohen:2018}.
Beyond clustering, axiomatic approaches have been widely adopted in other research areas including recommendation systems~\cite{pennock2000social, andersen2008trust}, computer vision~\cite{kenney2005axiomatic, chessel2006interpolating}, resource allocation~\cite{lan2010axiomatic}, and perhaps most famously in social choice theory~\cite{Arrow:1963}.

\section{Conclusion}
\label{sec:conclusion}

We present the first axiomatic analysis of node embeddings for featured networks.
Beyond establishing fundamental trade-offs inherent to the problem of node embedding, the impossibility result stated in \cref{thm:impossible} motivates a more thorough look at the nature of the node embedding problem.
In the context of clustering, \cite{Carlsson:2013} bypasses the impossibility results of \cite{Kleinberg:2002} by considering a relaxed codomain consisting of persistent clusterings, rather than simple partitions (see \cite[Section~1.3]{Riehl:2017} for a discussion on this).
Similarly, other works such as \cite{Ben:2008,Cohen:2018} show how clustering informed by a particular loss function can satisfy a similar set of axioms.
We envision the proposed framework motivating similar advances in the field of node embedding, ultimately redounding in better theoretical understanding and practical algorithms.

\section*{Acknowledgements}
This work was partially supported by NSF under award \href{https://www.nsf.gov/awardsearch/showAward?AWD_ID=2008555&HistoricalAwards=false}{CCF-2008555}.
Research was sponsored by the Army Research Office and was accomplished under Cooperative Agreement Number W911NF-19-2-0269.
The views and conclusions contained in this document are those of the authors and should not be interpreted as representing the official policies, either expressed or implied, of the Army Research Office or the U.S. Government.
The U.S. Government is authorized to reproduce and distribute reprints for Government purposes notwithstanding any copyright notation herein.
TMR was partially supported by the Ken Kennedy Institute 2020/21 Exxon-Mobil Graduate Fellowship.

\printbibliography

\newrefsection
\appendix

\section*{Appendix}

\section{A Categorical Variation of \cref{thm:impossible}}\label{app:categorical}

We frame the study of node embedding functions in terms of functors between categories, inspired by the work of~\cite{Carlsson:2013} on clustering.
In order to do so, let us first recall basic definitions in category theory.
We refer the reader to the books~\cite{MacLane:1971,Riehl:2017} for further reference.

\begin{defn}
A \emph{category} $\sC$ consists of:
\begin{enumerate}
    \item A collection of \emph{objects} $\ob(\sC)$
    \item For each $X,Y\in\ob(\sC)$, a collection of \emph{morphisms} $\mor_\sC(X,Y)$,
\end{enumerate}
such that
\begin{enumerate}
    \item For each $X\in\ob(\sC)$, there is a distinguished \emph{identity morphism} $\id_X\in\mor_\sC(X,X)$
    \item There is a \emph{composition map}, so that for any $X,Y,Z\in\ob(\sC)$,
    \begin{equation}
        \circ:\mor_\sC(X,Y)\times\mor_\sC(Y,Z)\to\mor_\sC(X,Z)
    \end{equation}
    where $\circ$ is associative
    \item For any $X,Y\in\ob(\sC), f\in\mor_\sC(X,Y), g\in\mor_\sC(Y,X)$, we have $\id_Y\circ f=g\circ\id_X$.
\end{enumerate}
\end{defn}
A type of morphism of particular interest is an isomorphism.
\begin{defn}
Let $\sC$ be a category, and let $X,Y\in\ob(\sC),f\in\mor_\sC(X,Y)$ be given.
If there exists $g\in\mor_\sC(Y,X)$ such that $g\circ f=\id_X, f\circ g=\id_Y$, we say that $f$ is an \emph{isomorphism}.
If there exists such an isomorphism, we say that $X$ and $Y$ are \emph{isomorphic} in $\sC$.
\end{defn}

The canonical example of a category is $\sSet$, where $\ob(\sSet)$ is the collection of sets, and $\mor_\sSet(X,Y)$ consists of all functions from $X$ to $Y$.
We also define the following relevant categories.

\begin{defn}
Let $\cS$ be a dissimilarity space.
The category of $\cS$-featured dissimilarity networks, denoted $\sN$, is such that
\begin{enumerate}
    \item $\ob(\sN)$ is the collection of all triples $(\cV,d,F)$, where $\cV$ is a finite set, $d$ is a dissimilarity function on $\cV$, and $F:\cV\to\cS$
    \item For $N_1,N_2\in\ob(\sN)$, $\mor_\sN(N_1,N_2)$ is the collection of all CoDNI maps from $N_1$ to $N_2$
    \item For $N\in\ob(\sN)$, $\id_N$ is defined in the obvious way, and composition is also defined as expected.
\end{enumerate}
\end{defn}

\begin{defn}
The category of finite pseudodissimilarity spaces, denoted $\sM$, is such that
\begin{enumerate}
    \item $\ob(\sM)$ is the collection of all tuples $(\cM,\rho)$, where $\cM$ is a finite set, and $\rho$ is a pseudodissimilarity function on $\cM$
    \item For $(M_1,\rho_1),(M_2,\rho_2)\in\ob(\sM)$, $\mor_\sM(M_1,M_2)$ is the collection of all non-expansive maps from $M_1$ to $M_2$, that is, maps $f:M_1\to M_2$ such that for all $x,y\in M_1$:
    \begin{equation}
        \rho_1(x,y)\geq\rho_2(f(x),f(y))
    \end{equation}
    \item For $M\in\ob(\sM)$, $\id_M$ is defined in the obvious way, and composition is also defined as expected.
\end{enumerate}
\end{defn}

One suspects that node embedding functions take objects in the category $\sN$, and yield corresponding objects in the category $\sM$.
Moreover, these maps should preserve structure in some way.
The way to describe this formally is via a functor:

\begin{defn}
Let $\sC,\sD$ be categories.
A \emph{functor} $F:\sC\to\sD$ consists of
\begin{enumerate}
    \item An \emph{object function} $F:\ob(\sC)\to\ob(\sD)$, whose application we denote by $FX$ for each $X\in\ob(\sC)$
    \item A \emph{morphism function}, so for every $X,Y\in\ob(\sC)$, $F:\mor_\sC(X,Y)\to\mor_\sD(FX,FY)$, whose application we denote by $Ff$ for each $f\in\mor_\sC(X,Y)$,
\end{enumerate}
such that
\begin{enumerate}
    \item For each $X\in\ob(\sC)$, we have $F\id_X=\id_{FX}$
    \item For each $X,Y,Z\in\ob(\sC), f\in\mor_\sC(X,Y), g\in\mor_\sC(Y,Z)$, we have $F(g\circ f)=Fg\circ Ff$.
\end{enumerate}
\end{defn}

Two functors can be composed by composing their respective object and morphism functions.
A particularly useful type of functor is the \emph{forgetful functor}.
For the category of featured dissimilarity networks $\sN$, the forgetful functor $\alpha:\sN\to\sSet$ is such that for any given dissimilarity networks $N_1=(\cV_1,d_1,F_1), N_2=(\cV_2,d_2,F_2)$ in $\ob(\sN)$ and CoDNI map $\Psi:\cV_1\to\cV_2$ in $\mor_\sN(N_1,N_2)$, we have
\begin{equation}
    \begin{aligned}
        \alpha:\ob(\sN)&\to\ob(\sSet) \\
        (\cV_1,d_1,F_1)&\mapsto \cV_1 \\
        \\
        \alpha:\mor_\sN(N_1,N_2)&\to\mor_\sSet(\cV_1,\cV_2) \\
        \Psi&\mapsto\Psi.
    \end{aligned}
\end{equation}
That is, the forgetful functor on $\sN$ discards the dissimilarity and feature information of a network, and preserves morphisms as maps between sets.
Similarly, we define the forgetful functor $\beta:\sM\to\sSet$ for the category of finite pseudodissimilarity spaces so that for any $(M_1,\rho_1),(M_2,\rho_2)\in\ob(\sM)$ and non-expansive map $f:M_1\to M_2\in\mor_\sN((M_1,\rho_1),(M_2,\rho_2))$, we have
\begin{equation}
    \begin{aligned}
        \beta:\ob(\sM)&\to\ob(\sSet) \\
        (M_1,\rho_1)&\mapsto M_1 \\
        \\
        \beta:\mor_\sM((M_1,\rho_1),(M_2,\rho_2))&\to\mor_\sSet(M_1,M_2) \\
        f&\mapsto f.
    \end{aligned}
\end{equation}

Before restating our impossibility result, we define the appropriate notion of graph-awareness in this context.
\begin{property}[Functorial graph-awareness]\label{prpt:cat-aware}
Let $\alpha:\sN\to\sSet$ and $\beta:\sM\to\sSet$ be the forgetful functors as previously described.
Let $\xi:\sN\to\sM$ be a functor such that $\beta\circ\xi=\alpha$.
Thus, for any network $N=(\cV,d,F)$, the finite pseudodissimilarity space $(M,\rho)=\xi N$ satisfies $M=\cV$;
in light of this, we simply write $(\cV,\rho)=\xi N$.
Consider a featured dissimilarity network $N=(\cV,d,F)$ on at least $3$ nodes, and consider an arbitrary pair of distinct nodes $i,j\in\cV$.
Let $\cD_{ij}$ be the set of dissimilarity functions $d'$ on $\cV$ such that $d'(i,j)=d(i,j)$, and for all $k,\ell\in\cV$ it holds that $d'(k,\ell)\geq d(k,\ell)$.
For any such $d'$, note that $\id_\cV\in\mor_\sN(N,(\cV,d',F))$.

If for all such networks $N$ and pairs of distinct nodes $i,j$ there exists a $d'\in\cD_{ij}$ such that putting $(\cV,\rho)=\xi N, (\cV,\rho')=\xi(\cV,d',F)$ yields $\rho(i,j)\neq\rho'(i,j)$, we say that the functor $\xi$ is \emph{graph-aware}.
\end{property}
Note that this definition of graph-awareness for functors is essentially identical to that for node embedding functions, except stated in a way that includes the condition of preserving the underlying finite set.

\begin{reptheorem}{thm:impossible}[Categorical Version]
There exists no \hyperref[prpt:cat-aware]{graph-aware} functor $\xi:\sN\to\sM$.
\end{reptheorem}

\begin{proof}
We prove the categorical version of \cref{thm:impossible} by reducing it to the original statement.
That is, we show that the existence of such a functor would imply the existence of a node embedding function that is self-contained, consistent, and graph-aware, which is absurd.
Let $\xi:\sN\to\sM$ be a functor.

We first consider the implications of the condition $\beta\circ\xi=\alpha$ stated in the redefinition of graph-awareness, which we state in~\cref{lem:cat-isembedding,lem:cat-selfcontained,lem:cat-consistent}.

\begin{lemma}\label{lem:cat-isembedding}
If $\beta\circ\xi=\alpha$, then the object function $\xi:\ob(\sN)\to\ob(\sM)$ is a node embedding function.
\end{lemma}

\begin{subproof}[Subproof]
The proof is trivial, by the definition of a node embedding function.
\end{subproof}

Since the condition $\beta\circ\xi=\alpha$ implies that the object function of $\xi$ is a node embedding function, it makes sense to speak of it being self-contained and consistent.

\begin{lemma}\label{lem:cat-selfcontained}
If $\beta\circ\xi=\alpha$, then the object function $\xi:\ob(\sN)\to\ob(\sM)$ is self-contained.
\end{lemma}

\begin{subproof}[Subproof]
Consider two dissimilarity networks $N_1=(\{i_1,j_1\},d_1,F_1), N_2=(\{i_2,j_2\},d_2,F_2)$, where $d_1(i_1,j_1)=d_2(i_2,j_2)$ and $F_1(i_1)=F_2(i_2), F_1(j_1)=F_2(j_2)$.
Then, the bijection
\begin{equation}
    \begin{aligned}
        f:\{i_1,j_1\}&\to\{i_2,j_2\} \\
        i_1&\mapsto j_1 \\
        i_2&\mapsto j_2
    \end{aligned}
\end{equation}
is an isomorphism.
Thus, the morphism $\xi f\in\mor_\sM(\xi N_1, \xi N_2)$ is also an isomorphism.
That is to say, the embeddings $(M_1,\rho_1)=\xi N_1, (M_2,\rho_2)=\xi N_2$ under $\xi$ are isomorphic in $\sM$.
By the hypothesis $\beta\circ\xi=\alpha$, we have that $\xi f=f$ and $M_1=\{i_1,j_2\}$ and $M_2=\{i_2,j_2\}$.
One can then check that the isomorphism of these spaces implies $\rho_1(i_1,j_1)=\rho_2(f(i_1),f(i_2))=\rho_2(i_2,j_2)$.
That is, the embedding distance between nodes in a two-node network is invariant under labeling of the nodes, thus satisfying self-containedness.
\end{subproof}

\begin{lemma}\label{lem:cat-consistent}
If $\beta\circ\xi=\alpha$, then the object function $\xi:\ob(\sN)\to\ob(\sM)$ is consistent.
\end{lemma}

\begin{subproof}[Subproof]
Let $N_1=(\cV_1,d_1,F_1), N_2=(\cV_2,d_2,F_2)\in\ob(\cN)$ be given such that there exists a CoDNI map $\Psi:\cV_1\to\cV_2$ in $\mor_\sN(N_1,N_2)$.
Put $(M_1,\rho_1)=\xi N_1, (M_2,\rho_2)=\xi N_2$.
By the hypothesis $\beta\circ\xi=\alpha$, we have $M_1=\cV_1, M_2=\cV_2$, and $\xi\Psi=\Psi$.
Since $\Psi\in\mor_\sM((M_1,\rho_1),(M_2,\rho_2))$, the following holds for all $i,j\in \cV_1$:
\begin{equation}
    \rho_1(i,j)\geq\rho_2(\Psi(i),\Psi(j)),
\end{equation}
so that the object function $\xi:\ob(\sN)\to\ob(\sM)$ satisfies consistency.
\end{subproof}

We now show that assuming graph-awareness of $\xi$ implies that the object function is graph-aware.
By the definition of a functor being graph-aware, $\beta\circ\xi=\alpha$, so that the object function of $\xi$ is a node embedding function by~\cref{lem:cat-isembedding}.
Thus, it makes sense to speak of the object function of $\xi$ being graph-aware as well.

\begin{lemma}\label{lem:cat-aware}
If $\xi$ is graph-aware, then the object function $\xi:\ob(\sN)\to\ob(\sM)$ is graph-aware.
\end{lemma}

\begin{subproof}[Subproof]
The proof is trivial, by the definition of graph-awareness for a node embedding function.
\end{subproof}

To conclude the proof, suppose for the sake of contradiction that $\xi$ is graph-aware.
Then, by \cref{lem:cat-isembedding,lem:cat-selfcontained,lem:cat-consistent,lem:cat-aware}, the object function $\xi:\ob(\sN)\to\ob(\sM)$ is a node embedding function that is self-contained, consistent, and graph-aware.
This is impossible, as desired.
\end{proof}

\section{Proof of \cref{prop:eigvec-fails}}\label{app:spectral}

As before, we verify this statement for all three properties.

\parheading{Self-containedness.}
Let a dissimilarity network $N_2=(\{i,j\},d,F)$ on two nodes is given.
Put $(\{i,j\},\phi)=\xi_C(N_2)$.
One can easily see that $\phi(i,j)=0$, satisfying self-containedness.

\parheading{Graph-awareness.}
Let a dissimilarity network $N=(\cV,d,F)$ on at least three nodes be given, and identify $\cV=\{1,2,\ldots,|\cV|\}$.
Pick two nodes in $\cV$, and identify them with the integers $1,2$, without loss of generality.
Put $\alpha=\kappa(d(1,2);F(1),F(2))$, and
\begin{equation}
    \beta=\argmin_{\substack{i,j\in\cV \\ j\notin\{1,2\}}}\kappa(d(i,j);F(i),F(j)).
\end{equation}
Consider the following $n\times n$ adjacency matrix:
\begin{equation}
    \mathbf{A}'=
    \begin{bmatrix}
        0 & \alpha & \beta & & & \\
        \alpha & 0 & \beta & & \mathbf{0} & \\
        \beta & \beta & 0  & & & \\
        & & & & & \\
        & \mathbf{0} & & & \mathbf{0} & \\
        & & & & & \\
    \end{bmatrix}.
\end{equation}
One can check that the first and second entries of the leading eigenvector of $\mathbf{A}'$, which we denote $\mathbf{u}'$, are equal, yielding an embedding distance of zero for the two corresponding nodes.
Alternatively, consider the following $n\times n$ matrix for the same values of $\alpha,\beta$:
\begin{equation}\label{eq:asymmetric-adjacency}
    \mathbf{A}''=
    \begin{bmatrix}
        0 & \alpha & \beta   & & & \\
        \alpha & 0 & \beta/2 & & \mathbf{0} & \\
        \beta & \beta/2 & 0  & & & \\
        & & & & & \\
        & \mathbf{0} & & & \mathbf{0} & \\
        & & & & & \\
    \end{bmatrix}.
\end{equation}
In this case, the first and second entries of the leading eigenvector of $\mathbf{A}''$, which we denote $\mathbf{u}''$, are not equal, yielding a nonzero embedding distance for the two corresponding nodes.

For any function $f:\cV\times\cV\setminus\{1,2\}\to\bbRp$ such that the range of $f$ is bounded by $\beta$ and $f(i,j)=0$ if and only if $i=j$, there exists a dissimilarity function $\widehat{d}\in\cD_{12}$ such that $f(i,j)=\kappa(\widehat{d}(i,j);F(i),F(j))$, due to the conditions on the function $\kappa$ and the choice of $\beta$.
Therefore, one can choose dissimilarity functions $\widehat{d}',\widehat{d}''\in\cD_{12}$ such that the adjacency matrix $\widehat{\mathbf{A}}'$ constructed from the dissimilarity network $(\cV,\widehat{d}',F)$ is arbitrarily close to $\mathbf{A}'$ in the Frobenius norm, and the adjacency matrix $\widehat{\mathbf{A}}''$ constructed from the dissimilarity network $(\cV,\widehat{d}'',F)$ is arbitrarily close to $\mathbf{A}''$ in the Frobenius norm.
By choosing such dissimilarity functions, one can find embeddings $(\cV,\widehat{\phi}')=\xi_C((\cV,\widehat{d}',F)),(\cV,\widehat{\phi}'')=\xi_C((\cV,\widehat{d}'',F))$ so that $\widehat{\phi}'(1,2)$ is arbitrarily close to zero, and $\widehat{\phi}''(1,2)$ is arbitrarily close to the nonzero dissimilarity obtained from the leading eigenvector of~\eqref{eq:asymmetric-adjacency}, due to the Davis-Kahan $\sin\theta$ Theorem~\cite{Davis:1970}.
It follows that the eigenvector centrality embedding procedure is graph-aware.

\parheading{Consistency.}
Choose an arbitrary $s\in\cS$, and consider a dissimilarity network $N=(\{1,2\},d,F)$, where $F(1)=F(2)=s$ and $d$ is such that $\kappa(d(1,2);s,s)=0.5$.
Letting $(\{1,2\},\phi)=\xi_C(N)$, we showed previously that $\phi(1,2)=0$.
Now consider a dissimilarity network $N'=(\{1,2,3\},d',F')$, where $F'(1)=F'(2)=F'(3)=s$, and $d'$ is chosen such that $d'(1,2)=d(1,2), \kappa(d'(1,3);s,s)=0.25, \kappa(d'(2,3);s,s)=0.125$, which is possible due to the properties of $\kappa$.
Letting $(\{1,2,3\},\phi')=\xi_C(N')$, one can check that $\phi'(1,2)\neq 0$.
Notice that the inclusion map $\Psi:\{1,2\}\hookrightarrow\{1,2,3\}$ is a CoDNI map from $N$ to $N'$, but $\phi(1,2)\leq\phi'(\Psi(1),\Psi(2))$, violating consistency. \hfill$\square$

\printbibliography

\end{document}